\documentclass[pmlr,twocolumn,10pt]{jmlr} 
\usepackage{colortbl}
\usepackage{stfloats} 

\usepackage{booktabs}
\usepackage{siunitx}

\usepackage[switch]{lineno}

\theorembodyfont{\upshape}
\theoremheaderfont{\scshape}
\theorempostheader{:}
\theoremsep{\newline}

\jmlrvolume{287}
\jmlryear{2025}
\jmlrsubmitted{} 
\jmlrpublished{} 
\jmlrworkshop{Conference on Health, Inference, and Learning (CHIL) 2025}

\title[Distributionally Robust Learning in Survival Analysis]{Distributionally Robust Learning in Survival Analysis}

\author{%
\Name{Yeping Jin}\Email{yepjin@bu.edu}\\
\addr Boston University, USA
\AND
\Name{Lauren Wise} \Email{lwise@bu.edu}\\
\addr Boston University, USA
\AND
\Name{Ioannis Ch. Paschalidis} \Email{yannisp@bu.edu}\\
\addr Boston University, USA
}

\usepackage[english]{babel}
\usepackage{mathtools}
\usepackage[utf8]{inputenc} 
\usepackage{hyperref}       
\usepackage{url}            
\usepackage{booktabs}       
\usepackage{amsfonts}       
\usepackage{nicefrac}       
\usepackage{microtype}      
\usepackage{xcolor}         
\usepackage{verbatim}
\usepackage{amsmath}
\usepackage{amssymb}
\usepackage[title]{appendix}
\usepackage{breqn}
\usepackage{amssymb,amsfonts}
\usepackage{stmaryrd}
\usepackage{booktabs} 
\usepackage{graphicx,caption,subcaption}
\usepackage{bm}
\graphicspath{ {./images/} }
\newcommand{\be}[1]{\begin{equation}\label{#1}}
\newcommand{\benon}{\begin{equation*}}  
\newcommand{\bemuln}[1]{\begin{multline}\label{#1}}
\newcommand{\bemul}{\begin{multline*}}
\newcommand{\bee}{\begin{eqnarray*}}
\newcommand{\eee}{\end{eqnarray*}}
\newcommand{\been}[1]{\begin{eqnarray}\label{#1}}
\newcommand{\eeen}{\end{eqnarray}}
\newcommand{\began}[1]{\begin{gather}\label{#1}}
\newcommand{\bega}{\begin{gather*}}
\newcommand{\bealn}[1]{\begin{align}\label{#1}}
\newcommand{\beal}{\begin{align*}}
\newcommand{\bealatn}[2]{\begin{alignat}{#1}\label{#2}}
\newcommand{\bealat}{\begin{alignat*}}
\newcommand{\bexalatn}[1]{\begin{xalignat}\label{#1}}
\newcommand{\bexalat}{\begin{xalignat*}}

\newcommand{\lb}{\llbracket}
\newcommand{\rb}{\rrbracket}

\DefineNamedColor{named}{BurntOrange}   {cmyk}{0,0.51,1,0}
\DefineNamedColor{named}{Red}           {cmyk}{0,1,1,0}
\DefineNamedColor{named}{Purple}        {cmyk}{0.45,0.86,0,0}
\DefineNamedColor{named}{NavyBlue}      {cmyk}{0.94,0.54,0,0}
\DefineNamedColor{named}{RoyalBlue}     {cmyk}{1,0.50,0,0}
\DefineNamedColor{named}{Blue}          {cmyk}{1,1,0,0}
\DefineNamedColor{named}{LightBlueGray} {cmyk}{0.15,0.09,0,0.09}
\DefineNamedColor{named}{Cerulean}      {cmyk}{0.94,0.11,0,0}
\DefineNamedColor{named}{Cyan}          {cmyk}{1,0,0,0}
\DefineNamedColor{named}{Emerald}       {cmyk}{1,0,0.50,0}
\DefineNamedColor{named}{JungleGreen}   {cmyk}{0.99,0,0.52,0}
\DefineNamedColor{named}{SeaGreen}      {cmyk}{0.69,0,0.50,0}
\DefineNamedColor{named}{Green}         {cmyk}{1,0,1,0}
\DefineNamedColor{named}{ForestGreen}   {cmyk}{0.91,0,0.88,0.12}
\DefineNamedColor{named}{PineGreen}     {cmyk}{0.92,0,0.59,0.25}
\DefineNamedColor{named}{LimeGreen}     {cmyk}{0.50,0,1,0}
\DefineNamedColor{named}{YellowGreen}   {cmyk}{0.44,0,0.74,0}
\DefineNamedColor{named}{SpringGreen}   {cmyk}{0.26,0,0.76,0}
\DefineNamedColor{named}{OliveGreen}    {cmyk}{0.64,0,0.95,0.40}
\DefineNamedColor{named}{RawSienna}     {cmyk}{0,0.72,1,0.45}
\DefineNamedColor{named}{Pink}     {cmyk}{0,0.71,0.35,0}

\usepackage{amsmath,amssymb}

\def\br{{\mathbf r}}

\def\bv{{\mathbf v}}

\def\bx{{\mathbf x}}  
\def\by{{\mathbf y}}
\def\bz{{\mathbf z}}

\def\texitem#1{\par\smallskip\noindent\hangindent 25pt
               \hbox to 25pt {\hss #1 ~}\ignorespaces}

\newcommand{\scrD}{\mathcal{D}}

\newcommand{\scrV}{\mathcal{V}}

\newcommand{\scrX}{\mathcal{X}}
\newcommand{\scrY}{\mathcal{Y}}
\newcommand{\scrZ}{\mathcal{Z}}

\newcommand{\bbeta}{\boldsymbol{\beta}}

\newcommand{\bzeta}{\boldsymbol{\zeta}}

\newcommand{\btheta}{\boldsymbol{\theta}}

\newcommand{\bpi}{{\boldsymbol{\pi}}}

\begin{document}
\maketitle

\begin{abstract}
    We introduce an innovative approach that incorporates a {\em Distributionally
      Robust Learning (DRL)} approach into Cox regression to enhance the robustness
    and accuracy of survival predictions. By formulating a DRL framework with a
    Wasserstein distance-based ambiguity set, we develop a variant Cox model that is
    less sensitive to assumptions about the underlying data distribution and more
    resilient to model misspecification and data perturbations. By leveraging
    Wasserstein duality, we reformulate the original min-max DRL problem into a
    tractable regularized empirical risk minimization problem, which can be computed
    by exponential conic programming. We provide guarantees on the finite sample
    behavior of our DRL-Cox model. Moreover, through extensive simulations and real
    world case studies, we demonstrate that our regression model achieves superior
    performance in terms of prediction accuracy and robustness compared with
    traditional methods.
\end{abstract}
\paragraph*{Data and Code Availability}
 A portion of the data we used is publicly available and is included in the code below. The remaining data is sourced from the Pregnancy Study Online (PRESTO), a web-based preconception cohort study \citep{wise_main}. Due to privacy considerations, these data cannot be shared publicly, as PRESTO participants did not provide informed consent for external data sharing. Our code is available at \url{https://github.com/noc-lab/drl_cox}.
\paragraph*{Institutional Review Board (IRB)} The research was conducted using de-identified data and was exempt from IRB approval. 

\section{Introduction} 
Survival analysis is a cornerstone in the field of medical research, providing
insights into the time until an event of interest, such as death or disease
recurrence. The Cox proportional hazards model \citep{Cox1972} has become a foundational tool in survival analysis, allowing researchers to assess the
impact of various covariates on survival time without the need to specify the
underlying survival distribution explicitly. Despite its widespread adoption, the
classical Cox model and its extensions often assume a known or fixed distribution of
covariates and model parameters. This assumption may result in significant
inaccuracies in survival predictions when faced with real-world data complexities,
such as departure from proportional hazards, data heterogeneity, and the presence of
outliers.

Recent advancements in optimization and statistical learning have introduced the
concept of {\em Distributionally Robust Learning (DRL)}, a methodology that seeks to
improve model performance under distributional uncertainty. DRL achieves this by
optimizing worst-case performance across a set of possible distributions defined
within an ambiguity set, thus offering a more resilient approach to modeling under
uncertainty. The integration of DRL into survival analysis, particularly the Cox
regression model, presents an opportunity to address some of the inherent limitations
of traditional survival analysis methods by enhancing their robustness against
uncertainties in the data distribution.

\subsection{Related Works} 
In recent years, DRL has undergone extensive development and become a powerful tool
for enhancing model robustness in the presence of uncertainty and distributional
shifts. DRL is designed to minimize worst-case loss over a set of possible distributions,
rather than relying on a single empirical distribution. Some of the earlier works \citep{Ben-Tal2009, dro2, dro3} laid the foundation for
broader applications in data-driven optimization, with extensions into machine
learning \citep{chen2018robust,EsfahaniK18, gao2023}. These works developed a
comprehensive DRL framework under the Wasserstein metric and applied this framework
onto various topics, such as nearest-neighbor estimation, semi-supervised learning
and reinforcement learning.

More recent developments have integrated DRL into more complex settings such as deep
learning and survival analysis. For example, \citet{EsfahaniK18} expanded DRL to settings with ambiguous data, optimizing models to handle uncertainties in
real-world applications. Moreover, \citet{DRL_ranking} introduced DRL into {\em
  Learning-To-Rank (LTR)} methods and constructed a new robust LTR model which
addresses the robustness issue by using the Wasserstein framework, effectively
mitigating the impact of noise, adversarial perturbations, and data contamination in
ranking tasks. These advancements demonstrate the growing importance of DRL in
creating more reliable, robust models, which has been widely applied into fields
where variability and data noise can significantly affect model performance, such as
finance \citep{finance1, finance2} and healthcare \citep{healthcare2, hao2023}.

On the other hand, model robustness has been a central focus in the field of survival
analysis for several decades, with researchers continually developing methods to
improve the resilience of survival models against challenges like outliers,
censoring, and distributional shifts. Over the years, techniques such as penalized
Cox \citep{pcox} and adaptive Lasso Cox \citep{10.1093/biomet/asm037} have been
introduced to enhance the robustness of Cox models, particularly in medical datasets
where prediction accuracy is critical for patient outcomes. These advancements aim to
ensure that survival models not only perform well under ideal conditions but also
maintain reliability in the presence of real-world data variability.

However, despite the wide application of DRL in various fields, its utilization in
survival analysis has been relatively rare. To our knowledge, there is only one series of
existing work \citep{HuC22} in this area, which used a KL divergence framework \citep{Duchi2018} to apply DRL on the Cox regression model. Although their approach
demonstrated improvements in fairness metrics, it does not guarantee robustness under outliers and perturbations when compared to the standard Cox
model. This highlights a major gap in the literature, as existing methods do not
fully address the challenges posed by noisy or highly variable datasets, which are
common in real-world applications such as healthcare. Therefore, there is a pressing
need for further exploration of DRL-based Cox models that not only improve fairness
but also ensure robustness to data irregularities, making survival predictions more
reliable and applicable to broader domains. Addressing this gap could lead to
substantial advancements in fields where robust survival predictions are critical,
including personalized medicine and risk assessment.

\subsection{Main Contribution} 
This paper aims to bridge the gap between robust optimization techniques and survival
analysis by proposing a DRL-enhanced Cox regression model with a Wasserstein-distance
induced ambiguity set. Our contributions are twofold. First, we introduce a novel
approach to survival analysis that leverages the strengths of DRL to mitigate the
effects of uncertain data distributions, thereby improving the reliability and
accuracy of survival estimates. By adjusting the Cox loss function to adapt the DRL
framework, we establish a well-defined min-max stochastic formulation for the DRL-Cox
model. We then use Wasserstein duality to derive an exponential conic program as a
tractable relaxation, where equality is reached if the feature space $\scrX$
coincides with $\mathbb{R}^d$.  Furthermore, we carry out a stochastic analysis of
the DRL-Cox model. We show that, with high probability, the optimal DRL loss acts as
an upper bound for the true expected loss, highlighting the model's favorable
behavior in finite sample scenarios. This suggests that the model maintains
robustness even in the presence of data contamination or noise.  Second, through a
series of simulations and real-world applications, we empirically demonstrate the
advantages of our proposed model over classical Cox regression techniques as well as
penalized Cox models, highlighting its potential to provide more robust and accurate
predictions in various settings. The numerical experiments demonstrate that our
DRL-Cox model exceeds the two prior Cox models by more than 5\% in terms of
predictive accuracy. This improvement highlights the effectiveness of incorporating
DRL in survival analysis, enhancing the model's ability to handle distributional
shifts and outliers more effectively than standard Cox models.

The remainder of this paper is organized as follows: Sec.~\ref{sec:prel} introduces
background material and presents a basic formulation of the DRL
problem. Sec.~\ref{sec:form} presents our main theoretical results, including a
tractable form of the DRL-enhanced Cox model and the associated optimization
strategy. After proving the main theorem, we also investigate the new model's finite
sample performance. Sec.~\ref{sec:exp} details the experimental setup, data
description, and presents results from our empirical analysis. Finally,
Sec.~\ref{sec:conc} presents the implications of our findings, potential limitations,
and directions for future research.

{\bf Notational Conventions:} Boldfaced lowercase letters denote vectors, boldface
upper case letters denote matrices, and ordinary lowercase letters denote
scalars. Given any integer $N>0$, we use $\lb N\rb$ to denote $\{1,\ldots,N\}$.  We
use $\mathbb{R}$ to represent the set of real numbers, and $\mathbb{R}^+$ the set of
non-negative real numbers. All vectors are column vectors and prime denotes
transpose. For space-saving reasons, we write $\bx=(x_1,\ldots,x_d)\in\mathbb{R}^d$.
Given $p\geq 1$ and $\bx\in\mathbb{R}^d$, $\|x\|_p:=(\sum_{i\in\lb
  d\rb}\bx_i^p)^{1/p}$ denotes the $\ell_p$ norm.  $\mathbb{E}_{\mathbb{P}}$ denotes
the expectation under a probability distribution $\mathbb{P}$.  For a dataset
$\scrD=\{\bz_1,\ldots,\bz_N\}$, we use $\hat{\mathbb{P}}_N$ to denote the empirical
measure supported on $\scrD$. In particular, $\hat{\mathbb{P}}_N = \frac{1}{N}
\sum_{i=1}^{N} \delta_{\bz_i}$, where $\delta_{\bz_i}$ denotes the Dirac delta
function at point $\bz_i$.

\section{Background Material} \label{sec:prel}

\subsection{Cox Regression Model}

The Cox regression model estimates a hazard function, representing the instantaneous
risk of an event occurring at a given time, while accounting for the effects of
covariates. In our setting, the training dataset is
\[
\scrD=\{\bz_i: i\in \lb N \rb\}=\{ (\bx_i, y_i, \zeta_i): i\in \lb N \rb \},
\]
where each data point $\bz_i\in \scrD$ has feature vectors $\bx_i\in
\scrX\subset\mathbb{R}^d$, observed duration time $y_i\in \scrY\subset\mathbb{R}^+$,
and event indicator $\zeta_i\in\{0,1\}$. If $\zeta_i=1$, then the event of interest
occurs after time $t=y_i$. Otherwise, $\zeta_i=0$ indicates that the event of
interest did not occur, and data are censored after time $t=y_i$. In this case, the
true duration time of point $\bz_i$ is unknown, but we have a lower bound $t\geq y_i$
instead. Standard Cox regression predicts the hazard rate with the
model
\[
h_{\bbeta}(\bx, t) = h_0(t) e^{\bbeta' \bx},
\]
where one computes an optimal $\bbeta$ to minimize the following loss:
\begin{equation}\label{cox loss}
L(\bx,\by,\bzeta,\bbeta) = \sum_{i=1}^N \zeta_i \left[\log \left(\sum_{j: y_j \geq
    y_i}\!e^{\bbeta' \bx_j}\!\right)\!-\!\bbeta' \bx_i\!\right], 
\end{equation}
where $\by=(y_1,\ldots,y_N)$ and $\bzeta=(\zeta_1,\ldots,\zeta_N)$. The above loss
function sums up the individual losses of every uncensored data point. For each
individual loss, we compute the log-sum-exponential function of all data points that
have longer duration times. It is noteworthy that $L$ is convex with regards to $\bx$
and $\bbeta$, but not with regards to $\by$, which is one of the main challenges for
DRL, as we will explain in Sec.~\ref{sec:loss}.

\subsection{DRL formulation for the Cox Model} \label{sec:form-cox}

DRL considers a range of possible distributions contained within a defined
uncertainty set, also known as the ambiguity set.  Defining the ambiguity set is
crucial, leading to the development of several distinct frameworks. This paper adopts the approach of \citet{chen2021}, wherein the ambiguity set is characterized
using the Wasserstein distance. The Wasserstein distance serves as a metric between
probability distributions as follows.

\begin{definition}[Wasserstein Distance]
    Consider two Polish (i.e., complete, separable, metric) probability spaces
    $(\scrV_1,\mathbb{P})$ and $(\scrV_2,\mathbb{Q})$ and a lower-semicontinuous cost
    function $s : \scrV_1\times \scrV_2 \rightarrow \mathbb{R}\cup\{\infty\}$. Then, the
    (order-1) Wasserstein distance can be defined as:
\begin{equation*}    W_s(\mathbb{P},\mathbb{Q})= \min_{\bpi} \int_{\scrV_1\times
    \scrV_2} s(\bv_1,\bv_2) d\bpi(\bv_1,\bv_2),
\end{equation*}
where $\bpi\in\mathcal{P}(\scrV_1\times \scrV_2)$ is a joint probability
  distribution of $\bv_1,\bv_2$ with marginals $\mathbb{P}$ and $\mathbb{Q}$.
\end{definition} 
\noindent The basic formulation of a DRL Cox problem can be expressed as follows:
\begin{equation}\label{original_DRO}
\min_{\bbeta \in \mathbb{R}^d} \sup_{\mathbb{P} \in \Omega_{\epsilon}} \mathbb{E}_{\mathbb{P}} [l(\bx,y,\zeta,\bbeta)],
\end{equation}
where $\bbeta$ are the coefficients associated with the covariates $\bx\in
\scrX\subseteq \mathbb{R}^d$, $(\bx,y,\zeta)$ represents the uncertain data point
with probability distribution $\mathbb{P}$, and $l(\bx,y,\zeta,\bbeta)$ is the
individual loss function of a single data point in Cox regression, which will be
determined in the next subsection. In (\ref{original_DRO}),
$\Omega_{\epsilon}:=\{\mathbb{P} : W_s(\mathbb{P}, \mathbb{P}_0) \leq \epsilon\}$ is
the set of all probability distributions that are within a Wasserstein ball (induced
by the metric $s$) of radius $\epsilon$, centered at a nominal distribution
$\mathbb{P}_0$, which will be taken to be the empirical distribution
$\hat{\mathbb{P}}_N$ induced by the training data set. By construction,
\eqref{original_DRO} ensures that the optimization solution is robust against
training distributions within the ambiguity set.

\subsection{Choice of Loss Function} \label{sec:loss}

One of the principal challenges encountered in this study lies in the inherent
structure of the standard Cox loss function. The DRL model requires that the loss
function depends only on individual data points, but the individual loss as
delineated in \eqref{cox loss} does not adhere to this prerequisite. Sample
splitting \citep{HuC22}, as an alternative strategy, has been proposed to address
this limitation. Despite its efficacy in enhancing fairness, this methodology does
not yield superior accuracy when subjected to data contamination, in comparison to analyses
performed with the original Cox model.

As a novel approach, this study proposes treating all training data as fixed
constants for the computation of the individual loss terms. Consequently, the
individual loss function is redefined as follows:
\begin{equation}\label{individual loss}
\!l(\!\bx,y,\zeta,\bbeta)=\!\zeta\!\left[\log\!\left(e^{\bbeta'\bx}+\!\sum_{j: y_j \geq
    y}\! e^{\bbeta' \bx_j}\right)\!-\!\bbeta' \bx\!\right]. 
\end{equation}
Compared to the standard Cox individual loss in \eqref{cox loss}, we introduce an additional exponential term in \eqref{individual loss}. This modification is intentional and aligns with the structure of the Cox proportional hazards model, where the partial likelihood inherently depends on pairwise comparisons across the dataset. By incorporating this term, we ensure that the loss function remains well-defined under the DRL framework in \eqref{original_DRO} while preserving its theoretical consistency with the proportional hazards assumption. Additionally, this adjustment guarantees the non-negativity of the loss for data points with long durations, providing stability in optimization. We acknowledge that this design deviates from the standard Cox loss; however, as demonstrated in our numerical experiments, this trade-off contributes to the robust performance of our approach. With this foundation, we now present our main result.

\section{Formulation of a DRL-Cox Model} \label{sec:form}

\subsection{Tractable form of DRL-Cox}

Given $N$ training data points $\{(\bx_i,y_i,\zeta_i)\}_{i\in \lb N \rb}$, we
    consider the empirical distribution they induce, which will serve as $\mathbb{P}_0$,
    the center of ambiguity set $\Omega_\epsilon$. Plugging the individual loss
    \eqref{individual loss} into \eqref{original_DRO}, we obtain the following
    stochastic program:
\begin{equation}
\label{detail_dro}
     \!\min_{\bbeta \in \mathbb{R}^d} \sup_{\mathbb{P} \in \Omega_{\epsilon}}\!\mathbb{E}_{(\bx,y,\zeta)\sim\mathbb{P}}\!\left[\!\zeta
\!\left(\log\!\left(\!e^{\bbeta'\bx}\!+\!\!\!\sum_{i:y_i\geq
         y}\! e^{\bbeta'\bx_i}\!\right)\!-\!\bbeta'\bx\!\right)\!\right]\!. 
\end{equation}
Evidently, Equation \eqref{detail_dro} presents challenges in terms of direct
solvability due to its inherent complexity. Consequently, it becomes imperative to
identify a tractable form that facilitates its estimation. To this end, we introduce
the main theorem of this study, which posits the following convex program as an upper
bound to Equation \eqref{detail_dro}, and we prove the theorem in two
phases. Initially, a strong dual formulation is constructed for the inner supremum
problem, followed by discretizing the non-convex $y-$direction. Subsequently, through
the application of the dual norm and the convex conjugate, a relaxation of the
supremum problem is attained. This process culminates in the computation of the
convex conjugate, which yields a tractable form of the original problem.

\begin{theorem}[Relaxation of DRL-Cox] Suppose the training data points
  $z_1,\ldots,z_N$ are bounded and sorted in decreasing order with regard to duration $y$, and
  the Wasserstein distance is induced by the $\ell_p$ norm. Let $(p,q)$ be H{\"o}lder
  conjugates, so that $p,q\geq 1$ and $\frac{1}{p}+\frac{1}{q}=1$. Then the following
  exponential conic program provides an upper bound for \eqref{detail_dro}:
\begin{align}\label{tractableDRO} 
\min_{\bbeta\in\mathbb{R}^d} &\qquad\quad\qquad \epsilon\|(\bbeta,\alpha)\|_q + \frac{1}{N} \sum_{i=1}^N \zeta_i s_i \\
\text{s.t.} &\  s_i \geq \log\!\left(e^{\bbeta'
  \bx_i}+ \hspace{-1pt} \sum_{i=1}^k e^{\bbeta'\bx_i}\right)\!-\!\bbeta' \bx_i -
\alpha\left(y_i-{y}_k\right),\nonumber\\ & \qquad\qquad\qquad\qquad\qquad\qquad\;\;\forall\; 1\leq i\leq k\leq N. \nonumber 
\end{align}
Moreover, if the covariate space satisfies $\scrX=\mathbb{R}^d$, then the optimal
cost of \eqref{tractableDRO} becomes equal to the value of the stochastic program
\eqref{detail_dro}.
\end{theorem}
\begin{proof}
We first introduce the following corollary to construct a strong dual for the inner
supremum of \eqref{detail_dro}. 
\begin{corollary}{\citep[Cor.2]{gao2023}}
Suppose that we use the empirical distribution
\[
\hat{\mathbb{P}}_N = \frac{1}{N} \sum_{i=1}^{N} \delta_{\bz_i},
\]
as the center of ambiguity set $\Omega_{\epsilon}$ and a Wasserstein distance induced
by the $\ell_p$ norm, where $\bz_i$, $i \in \lb N \rb$, are the observed realizations
of $\bz$. Then the primal problem $v_P=\sup_{\mathbb{Q} \in \Omega_{\epsilon}}
\mathbb{E}_{\mathbb{Q}}[l(\bz,\bbeta)]$ has a strong dual
    \begin{equation}\label{strongdual}
        \!v_P\!=\!\min_{\lambda\geq 0} \!\left\{\! \lambda\epsilon \!+\! \frac{1}{N}
        \sum_{i=1}^{N} \sup_{\bz\in \scrZ}[ l(\bz,\bbeta)\!-\!\lambda\|\bz-\bz_i\|_p]\!
        \right\}\!.\! 
    \end{equation}
\end{corollary}

Recall $\bz=(\bx,y,\zeta)$. Then the loss function becomes \begin{equation}\label{loss_l}
    l(\bz,\bbeta)=\zeta\left( \log\left(e^{\bbeta'\bx}+\sum_{i:y_i\geq
      y}e^{\bbeta'\bx_i}\right)-\bbeta'\bx\right),
\end{equation} 
and let $\scrZ$ the space of $\bz$'s. 
For each $i\in\lb N \rb$, we examine the $i$th supremum term in the right hand side of
\eqref{strongdual}:
 \begin{equation}\label{ori_s}
     \sup_{\bz\in \scrZ}[ l(\bz,\bbeta) - \lambda \|\bz-\bz_i\|_p]. 
 \end{equation}
If $\zeta_i=0$, $l(\bz,\bbeta)$ becomes $0$ by \eqref{loss_l}, so that the supremum
term in \eqref{ori_s} also becomes $0$ at the optimal solution $\bz^*=\bz_i$.
Therefore, we may set the $i$th supremum term in the right hand side of
\eqref{strongdual} to be $\zeta_i s_i$, where
\begin{equation}\label{sup_s}
s_i:=\sup_{(\bx,y)\in \scrX\times \scrY}\left\{l(\bx,y) -\lambda\|(\bx-\bx_i,y-y_i)\|_p\right\}.
\end{equation}
Note that in \eqref{sup_s}, we simplify the loss function \( l(\bx, y, \zeta, \bbeta)
\) to \( l(\bx, y) \) because \(\zeta_i\) has been extracted, and \(\bbeta\) remains
fixed within the supremum term:
\begin{align*}
l(\bx,y):&=l(\bx,y,1,\bbeta)\nonumber\\
     &=\log\left(e^{\bbeta'\bx}+\sum_{i:y_i\geq y}e^{\bbeta'\bx_i}\right)-\bbeta'\bx.
 \end{align*}
Our current goal is to simplify \eqref{sup_s}. Although the utilization of duality
could facilitate this process for convex loss functions $l(\bx,y)$, the loss function
$l$, as previously mentioned, exhibits a piece-wise constant nature with respect to
the $y$-direction, rendering it non-convex. However, note that $l(\bx,\cdot)$ is
monotonically decreasing, which implies that given fixed $\bx$, the function
\[
l(\bx,y)-\lambda\|(\bx-\bx_i,y-y_i)\|_p
\] 
will similarly decrease for $y \geq y_i$. Consequently, the supremum $s_i$ will be
attained when $y^* \leq y_i$. In fact, given the piece-wise constant characteristic
of $l(\bx,\cdot)$, the supremum will be achieved within the set $\{y_i, y_{i+1},
\ldots, y_n\}$. Therefore, we proceed to discretize $y$ in the following manner:
 \begin{equation}
\label{si}
    s_i=\max_{i\leq j\leq N}\sup_{\bx\in X}[l(\bx,y_j)-\lambda\|(\bx-\bx_i,y_j-y_i)\|_p].
\end{equation}
Since the non-convex variable $y$ has been discretized, we introduce two notational
simplifications. For $j\in \lb N \rb$, we may write the loss function as
\begin{equation}
    l_j(\bx):=l(\bx,y_j),
\end{equation}
and we may simplify the $j$th supremum term in $s_i$ as
\begin{equation}
    s_i^j:=\sup_{\bx\in X}[l_j(\bx)-\lambda\|(\bx-\bx_i,y_j-y_i)\|_p].
\end{equation}
In order to evaluate the supremum in \eqref{si}, we introduce the dual norm and the
convex conjugate. 

\begin{definition} 
  \begin{enumerate}
    \item {\rm (Dual Norm)} Given a norm $\|\cdot\|$, its dual norm $\|\cdot\|_*$ is defined as: 
    \begin{equation}\label{dual norm}
    \|\btheta\|_*:=\sup_{\|\bx\|\leq 1}\btheta'\bx.
    \end{equation}
    \item {\rm (Convex Conjugate)} Given a convex function $l$, its convex conjugate $l^*(\cdot)$ is:
    \begin{equation} 
        l^*(\btheta):=\sup_{\bx\in \textup{dom }l}\{\btheta'\bx-l(\bx)\}.
    \end{equation}
     In particular, we have $l(\bx)=\sup_{\btheta\in\Theta}[\btheta'\bx-l^*(\btheta)]$, where 
     \begin{equation}\label{big theta}
         \Theta:=\{\btheta: l^*(\btheta)<\infty\} 
     \end{equation}
     is the effective domain of the convex conjugate $l^*$.
  \end{enumerate}
\end{definition}

According to \eqref{dual norm}, the dual norm of $\ell_p$ norm is the $\ell_q$ norm
for $\frac{1}{p}+\frac{1}{q}=1$. On the other hand, since $l_j$ is convex, it is
equal to the convex conjugate of $l_j^*$. Therefore, we have the following relaxation
for $s_i^j$:
\begin{subequations}
\begin{align}
    &s_i^j\!=\!\sup_{\bx\in X}\sup_{\btheta\in\Theta}
    [\btheta'\bx-l_j^*(\btheta)-\lambda\|(\bx-\bx_i,y-y_i)\|_p]\nonumber \\
    &=\!\sup_{\bx\in X}\sup_{\btheta\in\Theta}\inf_{\|(\br,\alpha)\|_q\leq\lambda}\! 
    [\btheta'\bx\!-\!l_j^*(\btheta)\!+\!\br'(\bx\!-\!\bx_i)\!+\!\alpha(y\!-\!y_i)]\label{dual norm eq}\\
    &=\!\sup_{\btheta\in\Theta}\inf_{\|(\br,\alpha)\|_q\leq\lambda}\sup_{\bx\in X}\![(\btheta\!+\!\br)'\!\bx\!-\!l_j^*(\btheta)\!-\!\br'\!\bx_i\!+\!\alpha(y\!-\!y_i)]\label{duality eq} \\
    &\leq\!\sup_{\btheta\in\Theta}\inf_{\|(\br,\alpha)\|_q\leq\lambda}\sup_{\bx\in \mathbb{R}^d}\!
    [(\btheta\!+\!\br)'\!\bx\!-\!l_j^*(\btheta)\!-\!\br'\!\bx_i\!+\!\alpha(y\!-\!y_i)],
    \label{final ineq} 
\end{align}
\end{subequations}
where \eqref{dual norm eq} follows from the definition \eqref{dual norm}, and
\eqref{duality eq} is a direct result of the Minimax Theorem \citep{minimax}. The
inner supremum over $\bx\in\mathbb{R}^d$ in \eqref{final ineq} reaches $\infty$
unless $\br=-\btheta$. For all $j\in \lb N \rb$, set
\begin{equation}\label{kappaj}
    \kappa_j:=\sup\{\|\btheta\|_q:l_j^*(\btheta)<\infty\}.
\end{equation} 
If $\kappa_j>\lambda$, according to \eqref{big theta}, we can pick some
$\btheta\in\Theta$ such that $\|\btheta\|_q>\lambda$, which makes the inner supremum
attain $\infty$. Otherwise, $\kappa_j\leq \lambda$, and we cam take $\br=-\btheta$ to
achieve the inner infimum. Therefore,
\begin{align}\label{sij_relax}
    s_i^j&\leq\inf_{\|(-\btheta,\alpha)\|_q\leq\lambda}\sup_{\btheta\in\Theta}
    [-l_j^*(\btheta)+\btheta'\bx_i+\alpha(y-y_i)]\nonumber\\
    &=\inf_{\|(\btheta,\alpha)\|_q\leq\lambda}l_j(\bx_i)-\alpha(y_i-y).
\end{align}

Take $\kappa:=\max_{j\in \lb N \rb}\kappa_j$, then according to the relaxation in
\eqref{sij_relax} and the Minimax Theorem \citep{minimax}, for all
$\lambda\geq\kappa$, we obtain the following upper bound of \eqref{si}:
\begin{align*}
    s_i= \max_{i\leq j\leq N}s_i^j&\leq \max_{i\leq j\leq N}\inf_{\|(\btheta,\alpha)\|_q\leq\lambda}l_j(\bx_i)-\alpha(y_i-y_j)\\
    &=\inf_{\|(\btheta,\alpha)\|_q\leq\lambda}\max_{i\leq j\leq N}l_j(\bx_i)-\alpha(y_i-y_j).
\end{align*}
\noindent Therefore, by taking $\lambda=\|(\kappa,\alpha)\|_q$, we conclude the following relaxation of the primal problem \eqref{strongdual}: 
\begin{align}\label{vp}v_P&=\min_{\lambda\geq 0}\;\lambda\epsilon+\frac{1}{N}\sum_{i=1}^N\zeta_i s_i\nonumber\\
&\leq \epsilon\|(\kappa,\alpha)\|_q+\frac{1}{N}\sum_{i=1}^N\zeta_i\max_{i\leq j\leq
    N}[l(\bx_i,y_j)\!-\!\alpha(y_i\!-\!y_j)].   
\end{align}
Lastly, we aim to compute $\kappa$, which is the maximum of \eqref{kappaj}. To that
end, we consider the behavior of $l_j^*(\btheta)$. Given fixed $\btheta\in\Theta$
and $j\in \lb N\rb$, we denote $f(\bx)=\btheta'\bx-l_j(\bx)$,
which implies $l_j^*(\btheta) = \sup_{\bx\in \textup{dom }l} f(\bx)$. 
To maximize $f(\bx)$, we compute its gradient and Hessian matrix:
\begin{align*}\label{fx} 
    \!\nabla f(\bx)&=\btheta-\nabla\left(\log\left(e^{\bbeta'\bx}+\sum_{i=1}^j e^{\bbeta'\bx_i}\right)-\bbeta'\bx\right)\nonumber\\
&=\btheta+\frac{C}{e^{\bbeta'\bx}+C}\bbeta,\\
\nabla^2 f(\bx)&=
-\frac{e^{\bbeta'\bx}}{(e^{\bbeta'\bx}+C)^2}\bbeta\bbeta',\ \text{where}\ C =
\sum_{i=1}^j e^{\bbeta'\bx_i}. 
\end{align*}
Since the Hessian is negative semi-definite, $f(\bx)$ will reach the maximum if and
only if $\nabla f(\bx)$ reaches $0$. Therefore, for all $\btheta\in \mathbb{R}^d$,
$l^*_j(\btheta)<\infty$ if and only if exists $c\in (-1,0)$ such that
$\btheta=c\bbeta$. Therefore, we obtain $\kappa_j=\|\bbeta\|_*=\|\bbeta\|_q$ for
all $j$, and so is $\kappa$. Plugging $\kappa=\|\bbeta\|_q$ into \eqref{vp}, we
reach our desired relaxation of the stochastic program in \eqref{detail_dro}:
\begin{gather*}
     \inf_{\bbeta} \sup_{\mathbb{P} \in \Omega_{\epsilon}} E_{\mathbb{P}} [l(
       \bx,y,\zeta,\bbeta)]\leq\min_{\bbeta,\alpha}
     \epsilon\|(\bbeta,\alpha)\|_q+\frac{1}{N}\sum_{i=1}^N\zeta_i s_i, 
\end{gather*} 
where $s_i=\max_{i\leq j\leq N}[l(\bx_i,y_j,1,\bbeta)-\alpha(y_i-y_j)]$.
\end{proof}

\subsection{Performance Analysis}

Next, we analyze the performance of the DRL-Cox model. Our goal is to derive
stochastic guarantees that ensure the model's robustness and generalization
capability. While there exists analysis \citep{EsfahaniK18} on the asymptotic consistency of DRL, our study primarily focuses on the finite sample
performance, which is directly relevant to the numerical experiments presented in the
following section.

Given the empirical distribution \(\hat{\mathbb{P}}_N\) of $N$ observed data points
in the support of $\scrX\times \scrY\times \{0,1\}$, we denote the true measure by
\(\mathbb{P}^*\). Recall that \(\mathbb{P}^*\) is unknown since
\(\hat{\mathbb{P}}_N\) is contaminated. Therefore, we compute the DRL problem
\eqref{original_DRO} using \(\hat{\mathbb{P}}_N\) to implicitly optimize over the
true measure that is included in the ambiguity set with high confidence. Let
$\hat{J}_N$ and $\hat{\bbeta}_N$ be the optimal cost and optimal solution to the DRL
problem \eqref{original_DRO}, respectively. We thus have:
\begin{align*}
    \hat{J}_N = & \min_{\bbeta \in \mathbb{R}^d} \sup_{\mathbb{P} \in \Omega_{\epsilon}}
    \mathbb{E}_{\mathbb{P}} [l(\bx,y,\zeta,\bbeta)] \\
    = & \sup_{\mathbb{P} \in
      \Omega_{\epsilon}} \mathbb{E}_{\mathbb{P}} [l(\bx,y,\zeta,\hat{\bbeta}_N)], 
\end{align*}
where $\Omega_{\epsilon}=\{\mathbb{P} : W_s(\mathbb{P}, \hat{\mathbb{P}}_N) \leq
\epsilon\}$ is the Wasserstein ball defining the ambiguity set. To evaluate the
behavior of the optimal solution $\hat{\bbeta}_N$, we check if the expected loss
$\mathbb{E}_{\mathbb{P}^*} [l(\bx,y,\zeta,\hat{\bbeta}_N)]$ is bounded above by the
training loss $\hat{c}_N$. If this event occurs with high probability, we can assert
that $\hat{\bbeta}_N$ exhibits favorable finite sample behavior. In fact, we derive
the following lemma, which follows from a previous work \citep[Prop. 3]{zhao2015}.
\begingroup
  \tiny                  
  \setlength{\tabcolsep}{5pt}     
  \renewcommand{\arraystretch}{0.9}
\begin{table*}[b] 
    \centering
    \caption{Summary Statistics of the WHAS500 Dataset \citep{whas500}.}
    \label{tab:data_stats}
  \begin{tabular}{|c|c|c|c|c|c|c|c|c|c|c|c|c|c|}
    \hline
    & $\bx_1$ & $\bx_2$ & $\bx_3$ & $\bx_4$ & $\bx_5$ & $\bx_6$ & $\bx_7$ & $\bx_8$ & $\bx_9$ & $y$ & $\zeta$ \\
    \hline
    mean  &  69.846 &  26.614 &  78.266 &  87.018 &   6.116 & 144.704 &   0.156 &   0.750 &   0.400  &  882.436 &   0.430 \\
    \hline
    std   &  14.491 &   5.406 &  21.545 &  23.586 &   4.714 &  32.295 &   0.363 &   0.433 &   0.490 &  705.665 &   0.496 \\
    \hline
    min   &  30.000 &  13.045 &   6.000 &  35.000 &   0.000 &  57.000 &   0.000 &    0.000 &   0.000 &    1.000 &   0.000 \\
    \hline
    max   & 104.000 &  44.839 & 198.000 & 186.000 &  47.000 & 244.000 &   1.000 &  1.000 &   1.000 & 2358.000 &   1.000 \\
    \hline
\end{tabular}
    
\end{table*}
\endgroup
\begin{lemma}[Performance Guarantee]
    Suppose the data space $\scrX\times \scrY$ is bounded and finite, and the true
    measure $\mathbb{P}^*$ is finite. Given any significance level $\alpha\in (0,1)$,
    define the following threshold value:
    \[
    B(\alpha):= \sup\{\|\bz_1-\bz_2\|_p:\bz_1,\bz_2\in \scrX\times \scrY\}
    \sqrt{\frac{\log(1/\alpha)}{N}}.
    \] 
    Then for all ambiguity sets with radii $\epsilon\geq B(\alpha)$, we have the
    following robustness guarantee: 
    \begin{equation}\label{guarantee}
    \mathbb{P}(\mathbb{E}_{\mathbb{P}^*} [l(\bx,y,\zeta,\hat{\bbeta}_N)]\leq
    \hat{J}_N )\geq1-\alpha.
     \end{equation}
\end{lemma}

\section{Experiments} \label{sec:exp}

In this section, we present numerical experiments assessing the performance of our DRL-Cox model. The goal is to evaluate its robustness under the variability and imperfections of real-world data.

To ensure a comprehensive evaluation, we benchmark DRL-Cox against multiple survival analysis models, including the regular Cox model, the Ridge Cox model \citep{pcox}, the Lasso Cox model \citep{lassocox}, the Elastic Net Cox model \citep{elastic}, and the Sample-Splitting Cox model \citep{HuC22}. Beyond Cox-based models, we further include the Accelerated Failure Time (AFT) model \citep{aft} and the Random Survival Forest (RSF) model \citep{rsf}, ensuring a diverse set of methodologies for comparison.  

For consistency, all models are evaluated using the widely recognized concordance index (C-index) \citep{cindex} and the Integrated Time-Dependent AUC (iAUC) \citep{iauc}, ensuring a standardized and comprehensive metric for performance comparison.  

Since DRL-Cox introduces robustness through distributional considerations, a key factor influencing its performance is the choice of the ambiguity radius $\epsilon$. This parameter governs the balance between robustness and predictive accuracy. A small ambiguity radius results in an ambiguity set tightly concentrated around the empirical distribution, making the model nearly equivalent to the standard Cox model but with limited robustness to distributional shifts. Conversely, a large ambiguity radius increases conservatism, optimizing for the worst-case scenario over a broader range of distributions, which may lead to underestimation of risk scores and degraded performance in well-specified settings.

Recall that the performance guarantee in \eqref{guarantee} provides a theoretically justified lower bound for the ambiguity radius: for any significance level $\alpha$, choosing $\epsilon \geq B(\alpha)$ ensures that the expected loss under the true measure $\mathbb{P}^*$ does not exceed the empirical objective $\hat{J}_N$ with probability at least $1 - \alpha$. This threshold scales as $\mathcal{O}(\sqrt{\log(1/\alpha)/N})$, revealing a fundamental trade-off between robustness and sample efficiency. To refine the selection of $\epsilon$, we leverage the measure concentration method \citep{chen2021}, which bounds the deviation between empirical and true expectations to determine a statistically justified range of ambiguity radii. This prevents over-conservatism while maintaining robustness guarantees. To further fine-tune $\epsilon$, we apply cross-validation techniques, systematically evaluating different radii on validation sets to empirically select the optimal value that minimizes out-of-sample loss while preserving robustness against distributional shifts. This combined approach ensures that DRL-Cox achieves strong generalization with high probability.

Meanwhile, it is also noteworthy that the convex program in Eq.~\eqref{tractableDRO} encompasses $O(N^2)$ constraints, a factor which significantly impedes computational efficiency. In practical applications, a more feasible approach involves retaining only the constraints associated with $s_i$ in Eq.~\eqref{tractableDRO} for $i \leq k < i+\gamma$, thereby reducing the overall constraint count to $O(\gamma N)$. This modification not only enhances computational tractability but also maintains the integrity of the optimization process within a manageable scope. We then proceed to the following experiments with $\gamma=3$. In our experiments, the average runtime per trial is approximately 30 seconds, indicating that while the method remains computationally demanding, it is still feasible for practical applications.

For the rest of this section, we test the model on two distinct groups of data with different contamination patterns.

\begin{figure*}[h]
  \centering
    \caption{Performance metrics under varying distributional shifts for different models.}
  \includegraphics[width=\textwidth]{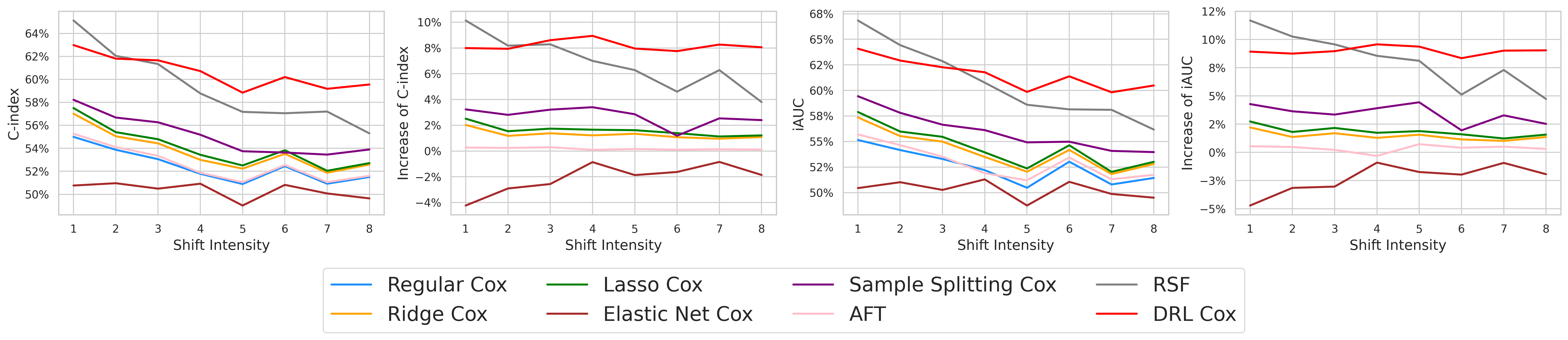}\\[1ex]
  \setlength\tabcolsep{0pt}
  \begin{tabular}{*{4}{p{0.24\textwidth}}}
    \centering (a) C-index 
    & \centering (b) C-index Increase Relative to Regular Cox
    & \centering (c) iAUC 
    & \centering (d) iAUC Increase Relative to Regular Cox \\
  \end{tabular}
  \label{fig:whas500_metrics}
\end{figure*}

\subsection{WHAS500 Dataset with Covariate Distributional Shift}
The WHAS500 dataset \citep{whas500} contains data on 500 patients hospitalized for acute myocardial infarction (heart attack) in Worcester, Massachusetts. It is widely used in survival analysis for time-to-event modeling, with a moderate censoring rate of 57\% and 14 features, ensuring balanced data utilization. The statistical summary of a subset of the features is presented in Table~\ref{tab:data_stats}.

We now introduce a distributional shift in its covariate. We define eight levels of shift intensity, progressively replacing 1 to 8 of these features with values sampled from a normal distribution. For each intensity level, we conduct 50 trials and compute the average C-index and iAUC to evaluate the model's performance under distributional shifts. We visualize the comparison of the two metrics in Figure~\ref{fig:whas500_metrics}. 

\begin{table*}[h]
\caption{Test scores for Miscarriage under different outlier severity levels with fixed outlier ratio (20\%). The best method for each outlier level is highlighted in bold.} \label{tab:sever_misc}
\centering
\footnotesize 
\rowcolors{2}{gray!15}{white} 
\resizebox{\textwidth}{!}{ 
\begin{tabular}{c|cc|cc|cc|cc|cc}
\toprule
  & \multicolumn{2}{c|}{Severity 1} & \multicolumn{2}{c|}{Severity 2} & \multicolumn{2}{c|}{Severity 3} & \multicolumn{2}{c|}{Severity 4} & \multicolumn{2}{c}{Severity 5} \\
 \hline
 Method & C-Index & iAUC & C-Index & iAUC & C-Index & iAUC & C-Index & iAUC & C-Index & iAUC \\
\hline
Regular Cox           & 0.5570 & 0.5727 & 0.5526 & 0.5675 & 0.5499 & 0.5657 & 0.5503 & 0.5644 & 0.5483 & 0.5646 \\
Ridge Cox             & 0.5622 & 0.5782 & 0.5607 & 0.5770 & 0.5568 & 0.5728 & 0.5595 & 0.5752 & 0.5599 & 0.5777 \\
Lasso Cox             & 0.5609 & 0.5713 & 0.5600 & 0.5650 & 0.5555 & 0.5637 & 0.5606 & 0.5598 & 0.5606 & 0.5609 \\
Elastic Net Cox       & 0.5623 & 0.5703 & \textbf{0.5657} & 0.5748 & 0.5614 & 0.5710 & 0.5618 & 0.5664 & 0.5634 & 0.5628 \\
Sample Splitting Cox  & 0.5328 & 0.5403 & 0.5359 & 0.5312 & 0.5159 & 0.5215 & 0.5085 & 0.5153 & 0.5077 & 0.5083 \\
AFT                   & 0.5565 & 0.5743 & 0.5511 & 0.5690 & 0.5483 & 0.5658 & 0.5461 & 0.5616 & 0.5464 & 0.5654 \\
RSF                   & 0.5587 & 0.5705 & 0.5545 & 0.5649 & 0.5618 & 0.5710 & 0.5662 & 0.5785 & 0.5626 & 0.5726 \\
DRL-Cox               & \textbf{0.5675} & \textbf{0.5819} & 0.5656 & \textbf{0.5812} & \textbf{0.5619} & \textbf{0.5770} & \textbf{0.5665} & \textbf{0.5808} & \textbf{0.5646} & \textbf{0.5803} \\
\bottomrule
\end{tabular}
} 
\end{table*}

\begin{table*}[h]
\caption{Test scores for Miscarriage under different outlier ratios, with fixed outlier severity (level 3). The best method for each outlier ratio is highlighted in bold.} \label{tab:ratio_misc}
\centering
\footnotesize 
\rowcolors{2}{gray!15}{white} 
\resizebox{\textwidth}{!}{ 
\begin{tabular}{c|cc|cc|cc|cc|cc|cc}
\toprule
 & \multicolumn{2}{c|}{5\% Outliers} & \multicolumn{2}{c|}{10\% Outliers} & \multicolumn{2}{c|}{15\% Outliers} & \multicolumn{2}{c|}{20\% Outliers} & \multicolumn{2}{c|}{25\% Outliers} & \multicolumn{2}{c}{30\% Outliers} \\
 \hline
 Method & C-Index & iAUC & C-Index & iAUC & C-Index & iAUC & C-Index & iAUC & C-Index & iAUC & C-Index & iAUC \\
\hline
Regular Cox           & 0.5613 & 0.5789 & 0.5553 & 0.5706 & 0.5544 & 0.5683 & 0.5580 & 0.5741 & 0.5521 & 0.5690 & 0.5540 & 0.5700 \\
Ridge Cox             & 0.5613 & 0.5770 & 0.5630 & 0.5788 & 0.5622 & 0.5771 & 0.5616 & 0.5779 & 0.5575 & 0.5751 & 0.5629 & 0.5805 \\
Lasso Cox             & 0.5614 & 0.5680 & 0.5612 & 0.5625 & 0.5609 & 0.5650 & 0.5620 & 0.5739 & 0.5578 & 0.5674 & 0.5613 & 0.5671 \\
Elastic Net Cox       & 0.5629 & 0.5717 & 0.5662 & 0.5764 & 0.5664 & 0.5768 & 0.5628 & 0.5719 & 0.5592 & 0.5693 & 0.5659 & 0.5724 \\
Sample Splitting Cox  & 0.5351 & 0.5395 & 0.5186 & 0.5276 & 0.5361 & 0.5479 & 0.5220 & 0.5290 & 0.5204 & 0.5343 & 0.5245 & 0.5265 \\
AFT                   & 0.5548 & 0.5712 & 0.5543 & 0.5710 & 0.5538 & 0.5690 & 0.5559 & 0.5732 & 0.5512 & 0.5712 & 0.5521 & 0.5706 \\
RSF                   & 0.5639 & 0.5742 & 0.5582 & 0.5677 & 0.5648 & 0.5759 & 0.5611 & 0.5700 & 0.5615 & 0.5707 & 0.5603 & 0.5700 \\
DRL-Cox               & \textbf{0.5687} & \textbf{0.5834} & \textbf{0.5679} & \textbf{0.5823} & \textbf{0.5677} & \textbf{0.5815} & \textbf{0.5684} & \textbf{0.5840} & \textbf{0.5639} & \textbf{0.5800} & \textbf{0.5675} & \textbf{0.5840} \\
\bottomrule
\end{tabular}
} 
\end{table*}

According to the results, DRL-Cox consistently outperforms all other models at high level of contamination, maintaining the highest C-index and iAUC values under the most intense shifts. The RSF model demonstrates moderate robustness, though it exhibits drastic decrease at intermediate and high shift levels. Meanwhile, the Sample-Splitting Cox model \citep{HuC22}, despite also incorporating a DRL framework, does not achieve robustness in this scenario. This outcome is expected, as its design primarily focuses on improving fairness in survival analysis \citep{HuC24} rather than optimizing predictive performance under distributional shift.

Moreover, it is also noteworthy that neither of the two test scores exhibit a strictly proportional relationship with distributional shift severity. Unlike traditional regression models where increasing noise typically leads to a consistent decline in performance, survival analysis methods such as the Cox model rely on pairwise comparisons of relative risk rather than absolute predictions. This structure makes them more resilient to certain types of distributional shifts, as long as the relative ranking of risk scores remains stable. In certain instances, models demonstrate a transient performance recovery at intermediate shift levels, implying that moderate degrees of distributional perturbation may introduce noise that remains within a tolerable range or even facilitates enhanced generalization. Specifically, if the introduced perturbations disrupt spurious correlations or regularize the model implicitly, they may help maintain or even improve predictive stability. This non-monotonic behavior underscores the intricate interplay between feature distributional shifts and model adaptability, suggesting that the impact of shift severity on predictive performance is highly dependent on the underlying data structure and is not governed by a simple linear trend.

\subsection{PRESTO Miscarriage Dataset with Outliers}

In addition to the public dataset, we incorporate real-world datasets from previous studies on predictive modeling of miscarriage \citep{pregnancy,miscarriage}. This dataset
is sourced from the {\em Pregnancy Study Online (PRESTO)}, a web-based preconception
cohort study by \citet{wise_main}. The Miscarriage dataset spans from 2013 to 2022
and involves 8,739 females aged 21--45 in the U.S. and Canada. This dataset comprises
189 predictive variables, and the censor rate is 79.64\%. We defined miscarriage as pregnancy loss before 20 completed
weeks of gestation.  In this dataset, the time to event was calculated as the
difference between the Exit week and the Start week (the gestational week at the time
of enrollment), using weeks as the time scale.

Given the high dimensionality and censoring rate of this dataset, training is inherently challenging and prone to the curse of dimensionality \citep{curse_dim}. Instead of introducing a distributional shift, we adopt a more conservative contamination approach by simulating outliers. Specifically, we contaminate the dataset by injecting outliers at varying ratios, ranging from 5\% to 30\%. To generate these outliers, we randomly select numerical features and perturb them with Gaussian noise of varying variance, thereby mimicking potential real-world perturbations encountered in survival analysis.

As we described above, our simulated noise depends on two parameters: the
proportion of outliers in the dataset (outlier ratios) and the degree of noise introduced
by the outliers (outlier severity). For each combination of these parameters, we implement five iterations and compare the averaged metrics for the benchmarks, illustrating how the
model's performance is affected under different levels of outliers contamination. The results are shown in Table~\ref{tab:sever_misc} and Table~\ref{tab:ratio_misc}.

Comparing to the distributional shifted dataset, the Miscarriage dataset with outliers presents a different dynamic. While all models
generally perform well, DRL-Cox exhibits a more pronounced improvement in C-index
compared to the baseline Cox model, especially in scenarios involving higher outlier
ratios and severity levels. The spikes in performance seen in the DRL-Cox model in these
conditions indicate its ability to maintain robust performance even in the face of
more irregular and complex data distributions, which are often present in
miscarriage-related datasets.

Overall, the above two experiments offer a direct visualization of the relationship
between varying levels of contamination and the model's predictive power, highlighting the
superior robustness of the DRL-Cox model compared to its competitors. The results demonstrate that under intense level of contamination, the DRL-Cox model consistently outperforms all other (robust) Cox models, as well as the AFT Model and RSF model, offering greater resilience and predictive
accuracy when handling real-world data irregularities, such as outliers and
perturbations, as seen in both datasets.

\section{Conclusion and Further Direction} \label{sec:conc}

In this article, we introduced a novel approach to enhance the robustness of Cox regression models through the application of DRL. By addressing inherent limitations in the standard Cox loss function, we have demonstrated the feasibility of deriving a more tractable form of the model. Our methodology, which leverages duality alongside the discretization of non-convex directions, culminates in a convex program that significantly enhances the model's robustness. The empirical evaluation reveals that the DRL-Cox model invariably outperforms other traditional survival methods.

In conclusion, this study's findings advocate for the integration of distributionally robust learning techniques in survival analysis, particularly in scenarios characterized by data uncertainty and noise. Despite the issue of computational complexity, the proposed DRL-Cox model represents an important step forward in developing more resilient and accurate predictive models. 

On the other hand, our approach inherently relies on the convexity of the individual loss function with respect to the $x$-direction, making direct extensions to deep survival models non-trivial. However, there exists a potential approach to incorporate DRL-based constraints into the weight matrices of deep survival models such as DeepHit \citep{deephit}, ensuring that the learned representations remain stable under perturbations. Such an approach has proven effective in various deep learning applications, including classification tasks and medical imaging \citep{dro_mlr,hao2023}, suggesting that DRL could also be integrated into deep survival learning frameworks. Future research could explore methods to adapt the DRL framework for deep survival models, potentially through structured approximations or alternative optimization techniques.

Additionally, further investigations into the bounds of ambiguity set adjustments and their implications on model performance across diverse datasets could yield deeper insights. Exploring alternative loss functions within the DRL framework may also lead to novel advancements in survival analysis and related fields. Through the advancements presented in this paper, we contribute to the broader discourse on the intersection of robust optimization and survival analysis, offering practical and theoretical insights that could pave the way for more sophisticated and resilient statistical models in medical research and beyond.

\bibliography{refs}

\end{document}